\documentclass{article}
\usepackage[margin = 1in]{geometry}
\usepackage{natbib}
\RequirePackage[OT1]{fontenc}
\RequirePackage{amsthm,amsmath}
\RequirePackage[colorlinks,citecolor=blue,urlcolor=blue]{hyperref}
\usepackage{graphicx, amssymb, mathrsfs,amsfonts,url, bm}
\usepackage{hyperref, cleveref, autonum}
\usepackage{color, soul}
\usepackage{enumerate}
\usepackage{multirow}
\usepackage{rotating} 
\usepackage{subfigure} 
\usepackage{cancel}

\usepackage{changes}
\usepackage{booktabs}
\usepackage{graphicx} 
\usepackage{xparse}
\usepackage[boxed]{algorithm2e}

\usepackage[T1]{fontenc}
\usepackage[font=small,labelfont=bf,tableposition=top]{caption}
\DeclareCaptionLabelFormat{andtable}{#1~#2  \&  \tablename~\thetable}

\theoremstyle{plain}
\theoremstyle{remark}         
\theoremstyle{definition} 

\NewDocumentCommand{\emphbf}{O{G}}{\emph{\textbf{#1}}}

\usepackage{amsmath}
\usepackage{amsthm} 
\usepackage{graphicx, amssymb, mathrsfs,amsfonts,url, bm}
\newtheorem{theorem}{Theorem}[section]

\newcommand{\balpha}{\boldsymbol\alpha}
\newcommand{\bbeta}{\boldsymbol\beta}

\newcommand{\bz}{\bm{z}}
\newcommand{\mathcalX}{\mathcal{X}}
\newcommand{\bpi}{\bm{\pi}}

\newcommand{\bxi}{\mathbf{x}_i}
\newcommand{\xknoi}{\mathcal{X}_{k,-i}}

\newcommand{\bznoi}{\mathbf{z}_{-i}}

\newcommand{\bmu}{\boldsymbol\mu}
\newcommand{\bsigma}{\boldsymbol\Sigma}

\newcommand{\gammadist}{\mathrm{Gamma}}

\author{Jun Lu\\
Computer Science, EPFL, Lausanne \\
\texttt{jun.lu.locky@gmail.com} \\
}

\title{Hyperprior on symmetric Dirichlet distribution}

\begin{document}
\maketitle 

\abstract{In this article we introduce how to put vague hyperprior on Dirichlet distribution, and we update the parameter of it by adaptive rejection sampling (ARS). Finally we analyze this hyperprior in an over-fitted mixture model by some synthetic experiments.}

\section{Introduction}
It has become popular to use over-fitted mixture models in which number of cluster $K$ is chosen as a conservative upper bound on the number of components under the expectation that only relatively few of the components $K^\prime$ will be occupied by data points in the samples $\mathcalX$. This kind of over-fitted mixture models has been successfully due to the ease in computation. 

Previously \citet{rousseau2011asymptotic} proved that quite generally, the posterior behaviour of overfitted mixtures depends on the chosen prior on the weights, and on the number of free parameters in the emission distributions (here $D$, i.e. the dimension of data). Specifically, they have proved that (a) If $\underline{\alpha}$=min$(\alpha_k, k \leq K)$>$D/2$ and if the number of components is larger than it should be, asymptotically two or more components in an overfitted mixture model will tend to merge with non-negligible weights. (b) In contrast, if $\overline{\alpha}$=max$(\alpha_k, k \leqslant K)<D/2$, the extra components are emptied at a rate of $N^{-1/2}$. Hence, if none of the components are small, it implies that $K$ is probably not larger than $K_0$. In the intermediate case, if min$(\alpha_k, k \leq K)\leq D/2 \leq$ max$(\alpha_k, k \leqslant K)$, then the situation varies depending on the $\alpha_k$'s and on the difference between $K$ and $K_0$. In particular, in the case where all $\alpha_k$'s are equal to $D/2$, then although the author does not prove definite result, they conjecture that the posterior distribution does not have a stable limit. 

\section{Hyperprior on symmetric Dirichlet distribution}\label{sec:hyper_fbgmm_background}

\begin{figure}[h!]
\centering
  \includegraphics[width=5cm]{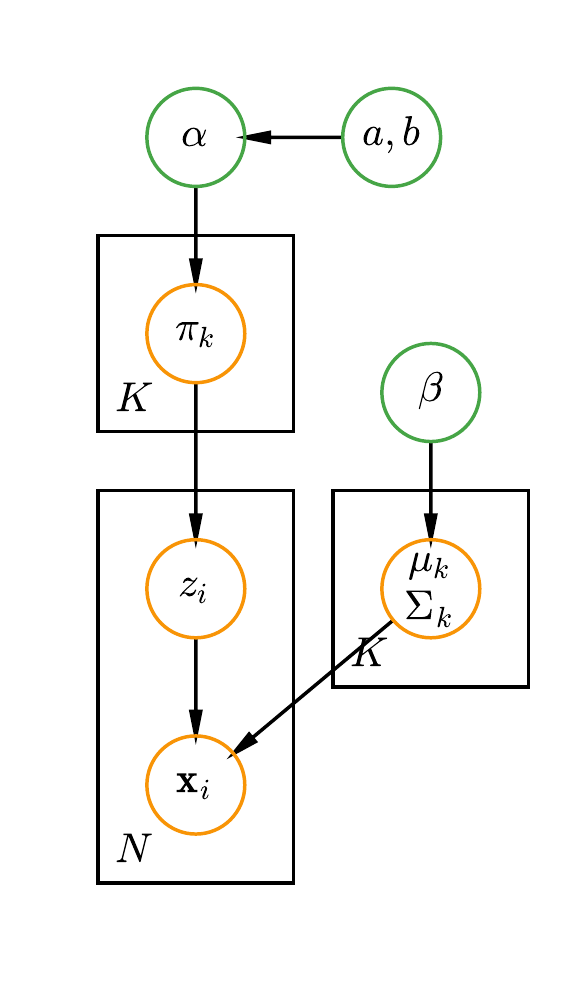}
  \caption{A Bayesian finite GMM with hyperprior on concentration parameter.}
  \label{fig:gmm_finite_with_hyper_second_one_same_with_first_one}
\end{figure}

Inspired by \citet{rasmussen1999infinite} and further discussed by \citet{gorur2010dirichlet}, they introduced a hyperprior on symmetric Dirichlet distribution prior. We here put a vague prior of Gamma shape on the concentration parameter $\alpha$ and use the standard $\balpha$ setting where $\alpha_k = \alpha = \alpha_+/K$ for $k=1, \ldots, K$. 

\begin{equation}
\alpha | a, b \sim \gammadist(a, b) \Longrightarrow p(\alpha | a,b) \propto \alpha^{a-1} e^{-b\alpha}. 
\end{equation}

To get the conditioned posterior distributions on $\alpha$ we need to derive the conditioned posterior distributions on all the other parameters,. But for a graphical model, this conditional distribution is a function only of the nodes in the Markov blanket. In our case, the Bayesian finite Gaussian mixture model, a Directed acyclic graphical (DAG) model, the Markov blanket includes the parents, the children, and the co-parents, as shown in Figure \ref{fig:gmm_finite_with_hyper_second_one_same_with_first_one}. From this graphical representation, we can find the Markov blanket for each parameter in the model, and then figure out their conditional posterior distribution to be derived:
\begin{equation}
\begin{aligned}
p(\alpha | \bpi, a, b) &\propto p(\alpha | a,b) p(\bpi | \alpha) \\
					&\propto \alpha^{a-1} e^{-b\alpha} \frac{\Gamma(K \alpha)}{\prod_{k=1}^K \Gamma(\alpha)} \prod_{k=1}^K \pi_k^{\alpha-1} \\
					&= \alpha^{a-1} e^{-b\alpha} (\pi_1 \ldots \pi_K)^{\alpha-1} \frac{\Gamma(K\alpha)}{[\Gamma(\alpha)]^K}. 
\end{aligned}
\end{equation}

The following two theorems give the proof that the conditional posterior distribution of $\alpha$ is log-concave.

\begin{theorem}
Define 
\begin{equation}
G(x) = \frac{\Gamma(K x)}{[\Gamma(x)]^K}. 
\label{equation:gx_log_concave2}
\end{equation}

For $x > 0$ and an arbitrary positive integer $K$,  the function G is strictly log-concave. 
\end{theorem}\label{theorem:gx_log_concave2}

\begin{proof}

From \citet{abramowitz1966handbook} we get $\Gamma(Kx) = (2\pi)^{\frac{1}{2}(1-K)} K^{K x-\frac{1}{2} } \prod_{i=0}^{K-1}\Gamma(x + \frac{i}{K})$.  Then

\begin{equation}
\log G(x) = \frac{1}{2}(1-K) \log(2\pi) + (Kx - \frac{1}{2})\log K + \sum_{i=0}^{K-1} \log \Gamma(x + \frac{i}{K}) - K \log \Gamma(x)
\end{equation}
and
\begin{equation}
[\log G(x)]^\prime =K \log K + \sum_{i=0}^{K-1} \Psi(x + \frac{i}{K}) - K \Psi(x),
\end{equation}
where $\Psi(x)$ is the Digamma function, and

\begin{equation}
\Psi^\prime(x)  = \sum_{h=0}^\infty \frac{1}{(x+h)^2}. 
\label{equation:fbgmm_hyperprior_digamma_derivative2}
\end{equation}

Thus
\begin{equation}
[\log G(x)]^{\prime \prime} = \left[\sum_{i=0}^{K-1} \Psi^\prime(x + \frac{i}{K}) \right]- K \Psi^\prime(x) < 0, \quad (x>0). 
\end{equation}
The last inequality comes from (\ref{equation:fbgmm_hyperprior_digamma_derivative2}) and concludes the theorem. 

\end{proof}

This theorem is a general case of Theorem 1 in \citet{merkle1997log}. 

\begin{theorem}
In $p(\alpha | \bpi, a, b)$, when $a \geq 1$, $p(\alpha | \bpi, a, b)$ is log-concave
\end{theorem}

\begin{proof}
It is easy to verify that $ \alpha^{a-1} e^{-b\alpha} (\pi_1 \ldots \pi_K)^{\alpha-1} $ is log-concave when $a \geq 1$. In view of that the product of two log-concave functions is log-concave and Theorem \ref{theorem:gx_log_concave2}, it follows that $\frac{\Gamma(K\alpha)}{[\Gamma(\alpha)]^K}$ is log-concave. This concludes the proof.
\end{proof}

From the two theorems above, we can find the conditional posterior for $\alpha$ depends only on the weight of each cluster. The distribution $p(\alpha | \bpi, a, b)$ is log-concave, so we may efficiently generate independent samples from this distribution using Adaptive Rejection Sampling (ARS) technique \citep{gilks1992adaptive}.

Although the proposed hyperprior on Dirichlet distribution prior for mixture model  is generic, we focus on its application in Gaussian mixture models for concreteness. We develop a collapsed Gibbs sampling algorithm based on \cite{neal2000markov} for posterior computation.

Let $\mathcalX$ be the observations, assumed to follow a mixture of multivariate Gaussian distributions. We use a conjugate Normal-Inverse-Wishart (NIW) prior $p(\bmu, \bsigma | \bbeta)$ for the mean vector $\bmu$ and covariance matrix $\bsigma$ in each multivariate Gaussian component, where $\bbeta$ consists of all the hyperparameters in NIW. A key quantity in a collapsed Gibbs sampler is the probability of each customer $i$ sitting with table $k$: $p(z_i = k | \bznoi, \mathcalX, \alpha, \bbeta)$, where $\bznoi$ are the seating assignments of all the other customers and $\alpha$ is the concentration parameter in Dirichlet distribution. This probability is calculated as follows:
\begin{equation} 
\begin{aligned}
p(z_i = k| \bznoi, \mathcal{X} , \alpha, \bbeta)  & \varpropto p(z_i = k | \bznoi, \alpha, \cancel{\bbeta})  p(\mathcal{X} |z_i = k, \bznoi, \cancel{\alpha}, \bbeta) \\
& = p(z_i = k| \bznoi, \alpha) p(\bxi |\mathcal{X}_{-i}, z_i = k, \bznoi, \bbeta) p(\mathcal{X}_{-i} |\cancel{z_i = k}, \bznoi, \bbeta)\\
& \varpropto p(z_i = k| \bznoi, \alpha) p(\bxi|\mathcal{X}_{-i}, z_i = k, \bznoi, \bbeta) \\
& \varpropto p(z_i = k| \bznoi, \alpha) p(\bxi | \xknoi, \bbeta), 										
\end{aligned}
\label{equation:sdir_fmm_collabsed_gibbs}
\end{equation}  
where $\xknoi$ are the observations in table $k$ excluding the $i^{th}$ observation. Algorithm~\ref{algo:sdir_fmm_plain_gibbs} gives the pseudo code of the collapsed Gibbs sampler to implement hyperprior for Dirichlet distribution prior in Gaussian mixture models. Note that ARS may require even 10-20 times the computational effort per iteration over sampling once from a gamma density and there is the issue of mixing being worse if we don’t marginalize out the $\pi$ in updating $\alpha$.  So this might have a very large impact on effective sample size (ESS) of the Markov chain. Hence, marginalizing out $\pi$ and using an approximation to the conditional distribution (perhaps with correction through an accept/reject step via usual Metropolis-Hastings or even just using importance weighting without the accept/reject) or even just a Metropolis-Hastings normal random walk for $\log(\alpha)$ may be much more efficient than ARS in practice. We here only introduce the updating by ARS.

\IncMargin{1em}
\begin{algorithm}
\SetKwData{Left}{left}\SetKwData{This}{this}\SetKwData{Up}{up}
\SetKwFunction{Union}{Union}\SetKwFunction{FindCompress}{FindCompress}
\SetKwInOut{Input}{input}\SetKwInOut{Output}{output}
\Input{Choose an initial $\bz$, $\alpha$ and $\bbeta$;}
\BlankLine

\For{$T$ iterations}{
\For{$i \leftarrow 1$ \KwTo $N$}{
	Remove $\bxi$'s statistics from component $z_i$ \;
	\For{$k\leftarrow 1$ \KwTo $K$}{
		Calculate $p(z_i=k| \bznoi, \alpha)$ \;
		Calculate $p(\bxi | \xknoi, \bbeta)$\;
		Calculate $p(z_i = k | \bznoi, \mathcal{X}, \alpha, \bbeta) \propto p(z_i=k| \bznoi, \alpha) p(\bxi | \xknoi, \bbeta)$\;
	}
	Sample $k_{new}$ from $p(z_i | \bznoi, \mathcalX, \alpha, \bbeta)$ after normalizing\;
	Add $\bxi$'s statistics to the component $z_i=k_{new}$ \;
}
$\star$ Draw current weight variable $\bpi = \{\pi_1, \pi_2, \ldots, \pi_K\}$ \;
$\star$ Update $\alpha$ using ARS\;
}
\caption{Collapsed Gibbs sampler for a finite Gaussian mixture model with hyperprior on Dirichlet distribution}\label{algo:sdir_fmm_plain_gibbs}
\end{algorithm}\DecMargin{1em}

\section{Experiments}
In the following experiments we evaluate the effect of a hyperprior on symmetric Dirichlet prior in finite Bayesian mixture model.

\subsection{Synthetic simulation}
The parameters of the simulations are as follows, where $K_0$ is the true cluster number. And we use $K$ to indicate the cluster number we used in the test:

Sim 1: $K_0=3$, with $N$=300, $\bm{\pi}$=\{0.5, 0.3, 0.2\}, $\bm{\mu}$=\{-5, 0, 5\} and $\bm{\sigma}$=\{1, 1, 1\};

In the test we put $\alpha \sim \gammadist(1, 1)$ as the hyperprior. Figure~\ref{fig:hyperdirichlet_sim1_traceplot_and_table} shows the result on Sim 1 with different set of $K$. Figure~\ref{fig:hyperdirichlet_sim1_alpha_several_k} shows the posterior density of $\alpha$ in each set of $K$. We can find that the larger $K-K_0$, the smaller the poserior mean of $\alpha$. This is what we expect, as the larger overfitting, the smaller $\alpha$ will shrink the weight vector in the edge of a probability simplex.

\begin{figure}[!ht]
    \centering
    \includegraphics[height = 4.4cm, width = 0.32\textwidth]{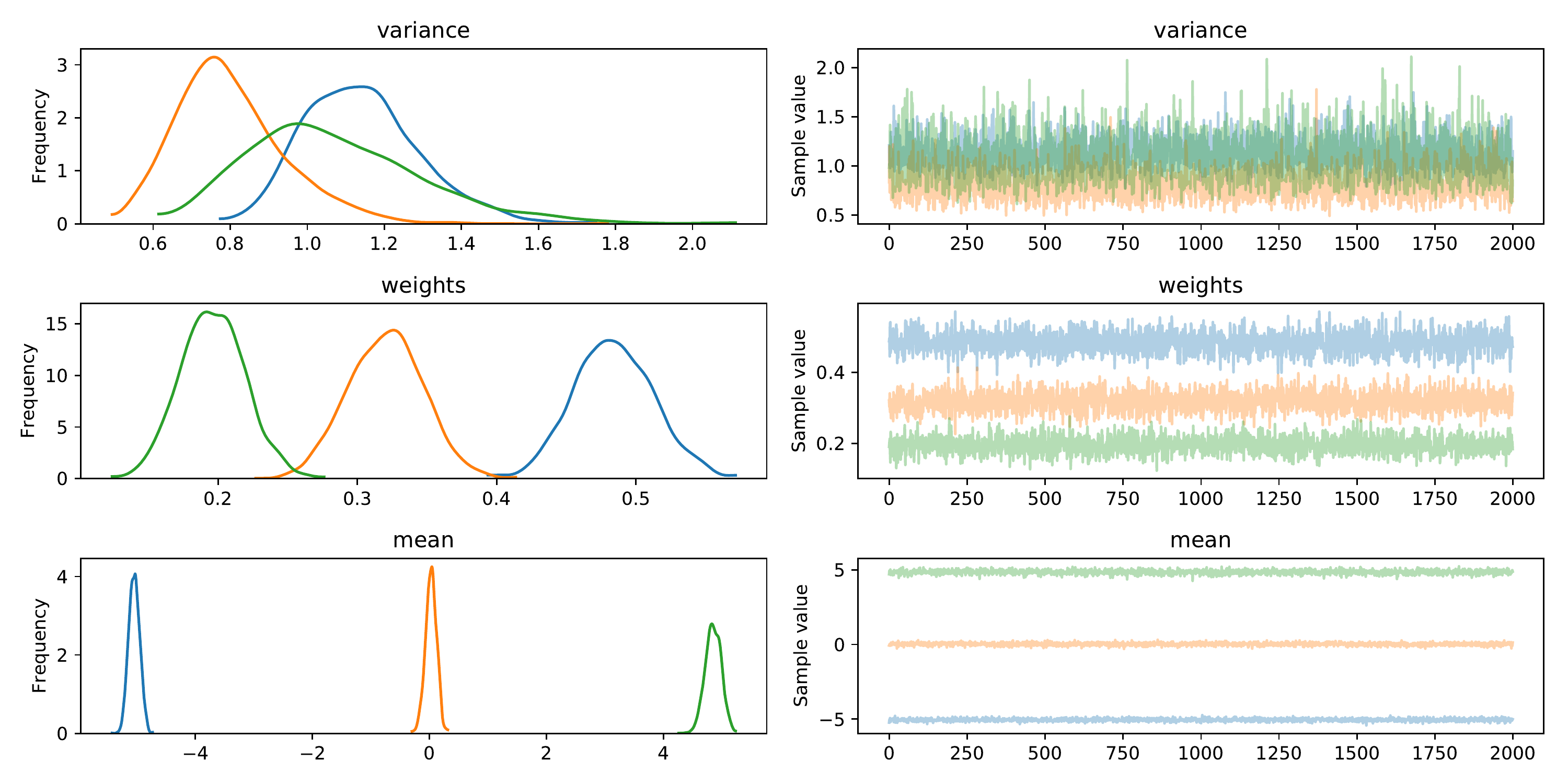}
    \qquad
\begin{tabular}[b]{cccccc}\hline
Sim 1 & \begin{tabular}[c]{@{}c@{}}NMI\\ (SE)\end{tabular}       & \begin{tabular}[c]{@{}c@{}}VI\\ (SE)\end{tabular}         & \begin{tabular}[c]{@{}c@{}}$\overline{K}$\\ (SE)\end{tabular} & \begin{tabular}[c]{@{}c@{}}$\overline{\alpha}$\\ (SE)\end{tabular} & $\overline{\pi}_{-}$ \\ \hline
$K=3$ & \begin{tabular}[c]{@{}c@{}}0.931\\ (8.5e-5)\end{tabular} & \begin{tabular}[c]{@{}c@{}}0.203\\ (2.5e-4)\end{tabular}  & \begin{tabular}[c]{@{}c@{}}3.0\\ (0.0)\end{tabular}           & \begin{tabular}[c]{@{}c@{}}1.84\\ (5.4e-3)\end{tabular}           & 0.0                  \\
$K=4$ & \begin{tabular}[c]{@{}c@{}}0.869\\ (2.0e-4)\end{tabular} & \begin{tabular}[c]{@{}c@{}}0.437\\ (7.7e-4)\end{tabular} & \begin{tabular}[c]{@{}c@{}}3.842\\ (2.6e-3)\end{tabular}      & \begin{tabular}[c]{@{}c@{}}1.43\\ (5.2e-3)\end{tabular}           & 0.08                 \\
$K=5$ & \begin{tabular}[c]{@{}c@{}}0.843\\ (2.7e-4)\end{tabular} & \begin{tabular}[c]{@{}c@{}}0.560\\ (11.1e-4)\end{tabular} & \begin{tabular}[c]{@{}c@{}}4.508\\ (3.2e-3)\end{tabular}      & \begin{tabular}[c]{@{}c@{}}1.01\\ (4.5e-3)\end{tabular}           & 0.12                 \\
$K=6$ & \begin{tabular}[c]{@{}c@{}}0.846\\ (3.3e-4)\end{tabular} & \begin{tabular}[c]{@{}c@{}}0.564\\ ( 14.9e-4)\end{tabular} & \begin{tabular}[c]{@{}c@{}}4.703\\ (5.2e-3)\end{tabular}      & \begin{tabular}[c]{@{}c@{}}0.618\\ (3.9e-3)\end{tabular}           & 0.11                
\label{table:hyperdirichlet_sim1_posterior_summary}
\end{tabular}
\captionlistentry[table]{A table beside a figure}
\captionsetup{labelformat=andtable}
\caption{\textbf{Left:} An example of traceplot of Gibbs sampling using hyperprior on Dirichlet distribution for variances, weights and means when $K=3$ in Sim 1 (Upper one: variance; Middle one: weights; Bottom one: mean). \textbf{Right:} Summary of posterior distribution in Sim 1: NMI is the normalized mutual information between true clustering and the resulting clustering. VI is the variation of information between true clustering and resulting clustering. SE is the standard error of mean.  $\overline{K}$ is the average occupied number of cluster, $\overline{\alpha}$ is the average $\alpha$ during sampling, $\overline{\pi}_{-}$ is the average of extra weight.}
\label{fig:hyperdirichlet_sim1_traceplot_and_table}
\end{figure}

\begin{figure}[!h]
\centering     
\subfigure[$K$=3]{\label{fig:sim1_alpha_k3}\includegraphics[width=0.34\textwidth]{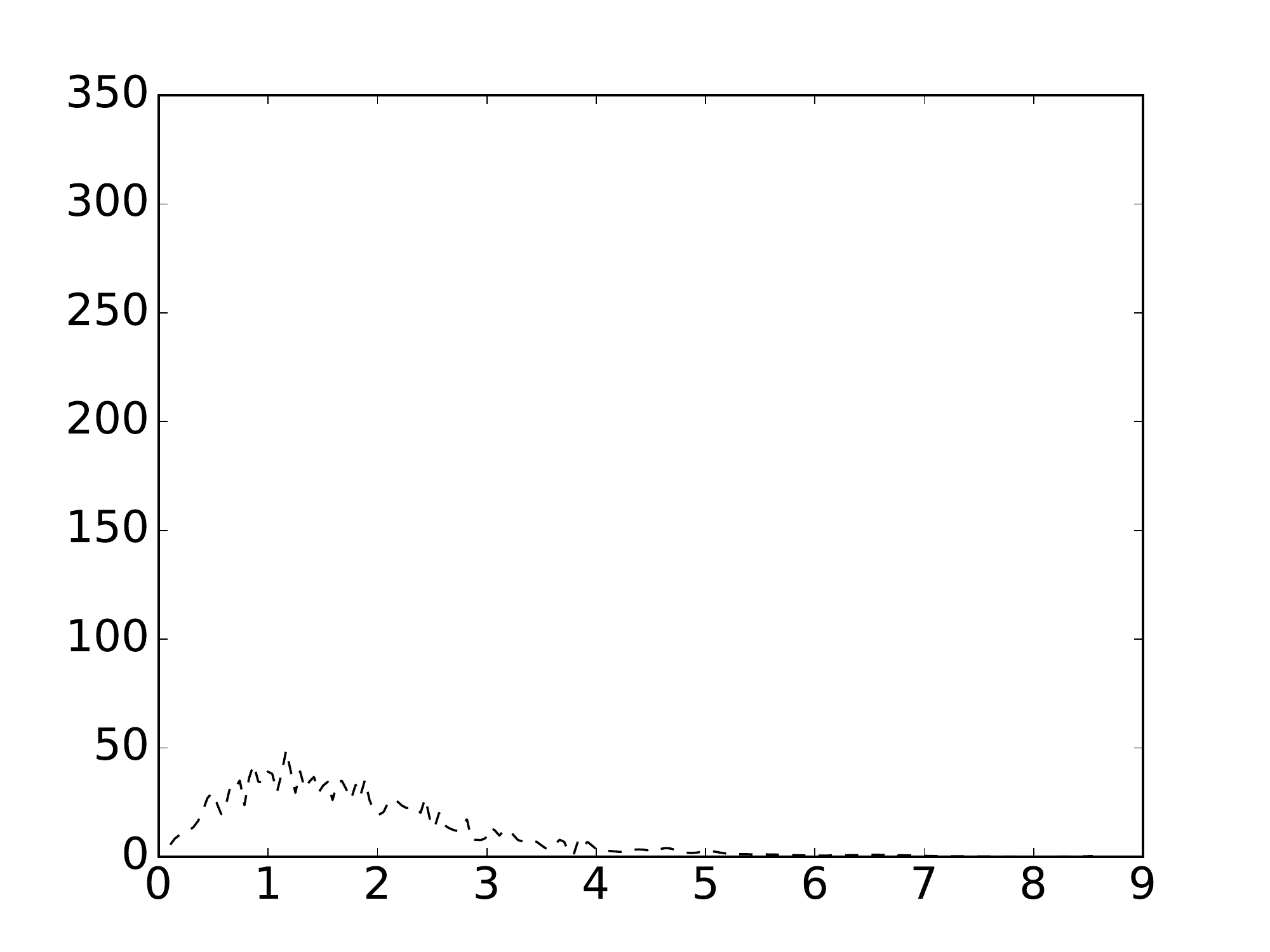}}
\subfigure[$K$=4]{\label{fig:sim1_alpha_k4}\includegraphics[width=0.34\textwidth]{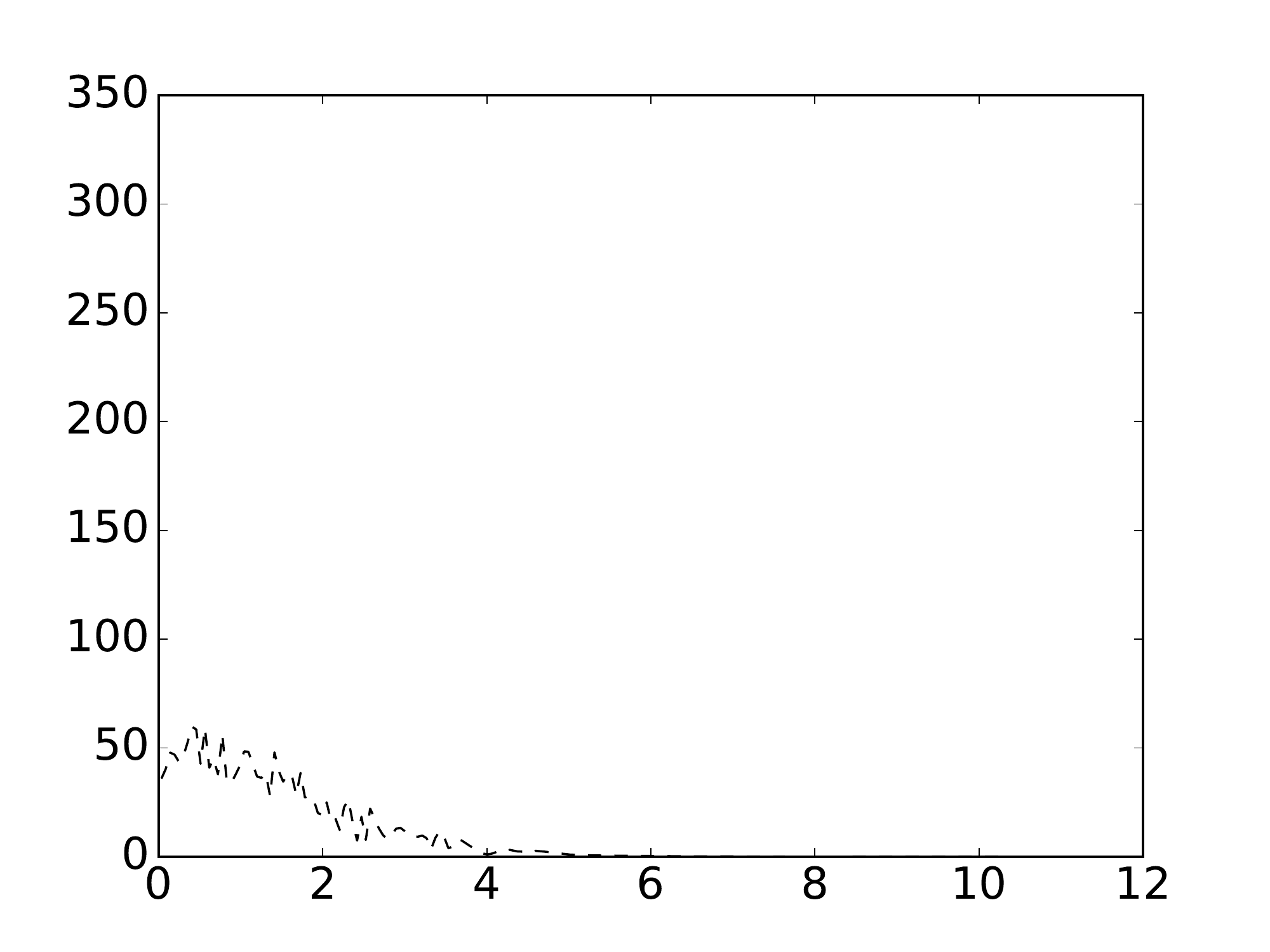}}
\subfigure[$K$=5]{\label{fig:sim1_alpha_k5}\includegraphics[width=0.34\textwidth]{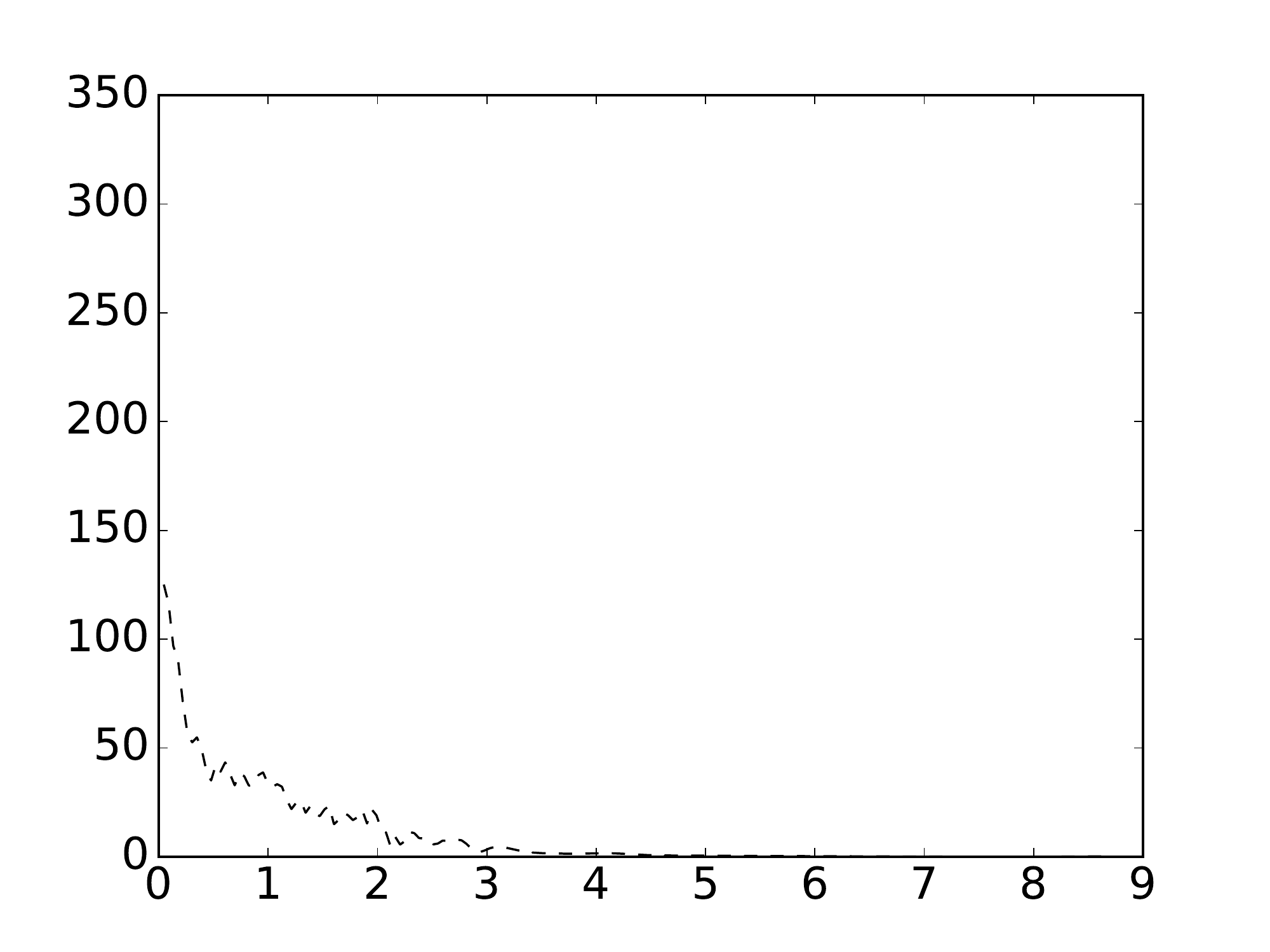}}
\subfigure[$K$=6]{\label{fig:sim1_alpha_k6}\includegraphics[width=0.34\textwidth]{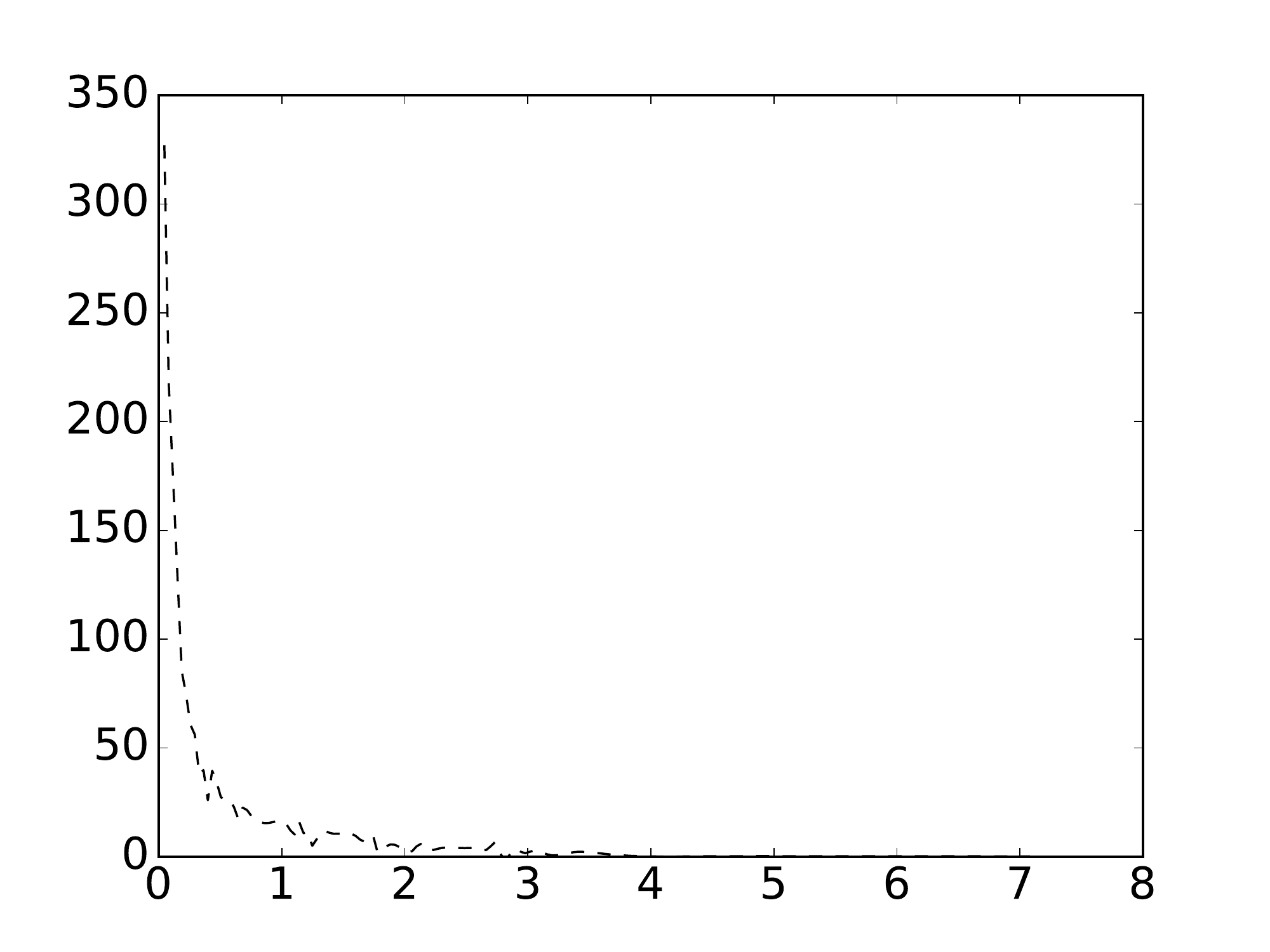}}
\caption{Posterior distribution for $\alpha$ in different overfitting settings in Sim 1.}
\label{fig:hyperdirichlet_sim1_alpha_several_k}
\end{figure}

\section{Conclusion}
We have proposed a new hyperprior on symmetric Dirichlet distribution in finite Bayesian mixture model. This hyperprior can learn the concentration parameter in Dirichlet prior due to over-fitting of the mixture model. The larger the overfitting (i.e. $K-K_0$ is larger, more overfitting), the smaller the concentration parameter. 

Although \citet{rousseau2011asymptotic} proved that $\overline{\alpha}$=max$(\alpha_k, k \leqslant K)<D/2$, the extra components are emptied at a rate of $N^{-1/2}$, it is still risky to use such small $\alpha$ in practice, for example, how much do we over-fit (i.e. how large the $K-K_0$). If $K-K_0$ is small, we will get very poor mixing from MCMC. Some efforts has been done further by \citet{van2015overfitting}. But simple hyperprior on Dirichlet distribution will somewhat release the burden.

\clearpage
\bibliography{bib}

\begin{thebibliography}{8}
\providecommand{\natexlab}[1]{#1}
\providecommand{\url}[1]{\texttt{#1}}
\expandafter\ifx\csname urlstyle\endcsname\relax
  \providecommand{\doi}[1]{doi: #1}\else
  \providecommand{\doi}{doi: \begingroup \urlstyle{rm}\Url}\fi

\bibitem[Abramowitz et~al.(1966)Abramowitz, Stegun,
  et~al.]{abramowitz1966handbook}
Milton Abramowitz, Irene~A Stegun, et~al.
\newblock Handbook of mathematical functions.
\newblock \emph{Applied mathematics series}, 55\penalty0 (62):\penalty0 39,
  1966.

\bibitem[Gilks \& Wild(1992)Gilks and Wild]{gilks1992adaptive}
Walter~R Gilks and Pascal Wild.
\newblock Adaptive rejection sampling for gibbs sampling.
\newblock \emph{Applied Statistics}, pp.\  337--348, 1992.

\bibitem[G{\"o}r{\"u}r \& Edward~Rasmussen(2010)G{\"o}r{\"u}r and
  Edward~Rasmussen]{gorur2010dirichlet}
Dilan G{\"o}r{\"u}r and Carl Edward~Rasmussen.
\newblock Dirichlet process gaussian mixture models: Choice of the base
  distribution.
\newblock \emph{Journal of Computer Science and Technology}, 25\penalty0
  (4):\penalty0 653--664, 2010.

\bibitem[Merkle(1997)]{merkle1997log}
Milan Merkle.
\newblock On log-convexity of a ratio of gamma functions.
\newblock \emph{Publikacije Elektrotehni{\v{c}}kog fakulteta. Serija
  Matematika}, pp.\  114--119, 1997.

\bibitem[Neal(2000)]{neal2000markov}
Radford~M Neal.
\newblock Markov chain sampling methods for {D}irichlet process mixture models.
\newblock \emph{Journal of Computational and Graphical Statistics}, 9\penalty0
  (2):\penalty0 249--265, 2000.

\bibitem[Rasmussen(1999)]{rasmussen1999infinite}
Carl~Edward Rasmussen.
\newblock The infinite gaussian mixture model.
\newblock In \emph{Advances in Neural Information Processing Systems},
  volume~12, pp.\  554--560, 1999.

\bibitem[Rousseau \& Mengersen(2011)Rousseau and
  Mengersen]{rousseau2011asymptotic}
Judith Rousseau and Kerrie Mengersen.
\newblock Asymptotic behaviour of the posterior distribution in overfitted
  mixture models.
\newblock \emph{Journal of the Royal Statistical Society: Series B (Statistical
  Methodology)}, 73\penalty0 (5):\penalty0 689--710, 2011.

\bibitem[van Havre et~al.(2015)van Havre, White, Rousseau, and
  Mengersen]{van2015overfitting}
Zo{\'e} van Havre, Nicole White, Judith Rousseau, and Kerrie Mengersen.
\newblock {Overfitting Bayesian mixture models with an unknown number of
  components}.
\newblock \emph{PloS one}, 10\penalty0 (7):\penalty0 e0131739, 2015.

\end{thebibliography}
\bibliographystyle{epflbibstyle}
\end{document}